\documentclass[10pt,twocolumn]{article}
\usepackage[
left=2.00cm,
right=2.00cm,
top=2.00cm,
bottom=2.00cm
]{geometry}
\usepackage{titlesec}
\titleformat{\section}[hang]
{\large\bfseries}
{\thesection.}{0.5em}{}

\titleformat{\subsection}[hang]
{\bfseries}
{\thesubsection.}{0.5em}{}

\usepackage{footnote}
\makesavenoteenv{tabular}
\makesavenoteenv{table}

\usepackage[round]{natbib}

\usepackage{color}
\usepackage{amsmath,amssymb,amsthm,bbm}
\usepackage{algorithm, algorithmic}
\usepackage{microtype}
\usepackage{graphicx}
\usepackage[titletoc]{appendix}
\usepackage{subfigure}
\usepackage{booktabs} 
\usepackage{enumitem}

\newcount\comments  
\definecolor{darkgreen}{rgb}{0,0.6,0}
\newcommand{\genComment}[2]{\ifnum\comments=1{\color{#1}{\textsf{\footnotesize{#2}}}}\fi}

\newcommand{\E}{\mathbb{E}}
\newcommand{\prob}{\mathbb{P}}

\newcommand{\Tmix}{T^\text{mix}}

\newcommand{\bigO}{\mathcal{O}}
\newcommand{\Pactive}{P^\text{active}}
\newcommand{\Ppassive}{P^\text{passive}}
\newcommand{\thetaStar}{\theta^{\star}}
\newcommand{\piStar}{\pi^{\star}}
\newcommand{\sumtt}{\sum_{t=1}^{T}}
\newcommand{\ind}{\mathbbm{1}}

\newcommand{\Hc}{\mathcal{H}}
\newcommand{\phat}{\hat{p}}

\newtheorem{theorem}{Theorem}
\newtheorem{condition}{Condition}

\newtheorem{lemma}[theorem]{Lemma}

\newtheorem{proposition}[theorem]{Proposition}
\newtheorem{assumption}{Assumption}

\title{Thompson Sampling in Non-Episodic Restless Bandits}

\author{
  Young Hun Jung\\
  University of Michigan\\
  \texttt{yhjung@umich.edu} \\
  \and
  Marc Abeille\\
  Criteo AI Lab\\
  \texttt{m.abeille@criteo.com}\\
  \and
  Ambuj Tewari\\
  University of Michigan\\
  \texttt{tewaria@umich.edu} \\
}

\begin{document}
\twocolumn[
\maketitle]
\begin{abstract}
\textit{Restless bandit} problems assume time-varying reward distributions of the arms, 
which adds flexibility to the model
but makes the analysis more challenging. 
We study learning algorithms over the unknown reward distributions and prove a sub-linear, $\bigO(\sqrt{T}\log T)$, regret bound for a variant of Thompson sampling. 
Our analysis applies in the \textit{infinite time horizon} setting, 
resolving the open question raised by
\citet{jung2019regret} whose analysis is limited to the episodic case.
We adopt their policy mapping framework, which allows our algorithm to be efficient and simultaneously keeps the regret meaningful. 
Our algorithm adapts the TSDE algorithm of \citet{ouyang2017learning} in a non-trivial manner to account for the special structure of restless bandits. 
We test our algorithm on a simulated dynamic channel access problem with several policy mappings, and the empirical regrets agree with the theoretical bound regardless of the choice of the policy mapping.
\end{abstract}


\section{INTRODUCTION}
In contrast to the classical multi-armed bandits (MABs), 
\textit{restless multi-armed bandits} (RMABs), introduced by \citet{whittle1988restless}, 
assume reward distributions that change along with the time. 
Due to their non-stationary nature, 
RMABs can model more complicated systems and thus get more attention in practice and theoretical literature.
In practice, they are used in a wide spectrum of applications including \textit{sensor management} (Chp. 7 in \citet{hero2007foundations} and Chp. 5 in \citet{biglieri2013principles}), \textit{dynamic channel access} problems \citep{liu2011logarithmic,liu2013learning}, and \textit{online recommendation systems} \citep{meshram2017restless}.
Theoretically, a variety of research communities have contributed to the literature on restless bandits, e.g., \textit{complexity theory} \citep{blondel2000survey},
\textit{applied probability} \citep{weber1990index},
and \textit{optimization} \citep{bertsimas2000restless}.


In this setting, there are $K$  independent arms indexed by $k \in [K]$.\footnote{For an integer $n$, we denote the set $\{1, \cdots, n\}$ by $[n]$.} Each arm is characterized by an internal state $s_k \in S_k$ which evolves in a \textit{Markovian fashion} according to the (possibly distinct) transition matrices $\Pactive_k$ and $\Ppassive_k$ depending on whether the arm is pulled (i.e., \textit{active}) or not (i.e., \textit{passive}). 
The reward of pulling an arm $k$ depends on its state $s^t_k$, which brings the non-stationarity.

We aggregate the transition matrices as
$\theta \in \Theta$ and consider this problem as a \textit{Reinforcement Learning} problem where $\theta$ is unknown to the learner. 
This problem has a complication in defining the baseline competitor against which the learner competes. 
It is not guaranteed, without additional assumptions, that the optimal policy exists,
and even if it exists, 
\citet{papadimitriou1999complexity} show that it is generally PSPACE hard to compute the optimal policy.

Researchers take different paths to tackle this challenge. 
Some define the regret using a simpler policy, which can be easily computed (e.g., see \citet{tekin2012online, liu2013learning}). 
They compare the learner's reward to a policy that pulls a fixed set of arms every round. 
Their algorithm is efficient and has a strong regret guarantee, $\bigO(\log T)$, 
but this baseline policy is known to be weak in the RMAB setting, which makes the regret less meaningful. 
Our empirical results in Sec. \ref{sec:experiments} also show the weakness of this policy.
Another breakthrough is made by \citet{ortner2012regret} who show a sub-linear regret bound against the optimal policy. 
However, they ignore the computational burden of their algorithm. 

\citet{jung2019regret} propose another interesting direction in that they introduce a deterministic \textit{policy mapping} $\mu$. 
It takes the system parameter $\theta$ as an input and outputs a deterministic stationary policy $\pi = \mu(\theta)$. 
Then the learner competes against the policy 
$\piStar = \mu(\thetaStar)$, where $\thetaStar$ denotes the true system. 
This framework is general enough to include the best fixed arm policy and the optimal policy that are mentioned earlier. 
That being said, one can achieve an efficient algorithm by choosing an efficient mapping $\mu$ or 
make the regret more meaningful with a stronger policy. 
In fact, there are different lines of work (e.g., \citet{whittle1988restless, liu2010indexability, meshram2017restless}) that study an efficient way, namely the \textit{Whittle index policy}, to approximate the optimal algorithm. 
Using this policy as a mapping, one can obtain an efficient algorithm with a meaningful regret simultaneously.

In this paper, we also adopt the policy mapping from \citet{jung2019regret} and answer an open question raised by them. 
Specifically, they prove the regret bound of \textit{Thompson sampling} in the episodic restless bandits where the system periodically resets. 
From the episodic assumption, the problem boils down to a \textit{finite horizon} problem, which makes the analysis simpler. 
However, there are many cases (e.g., online recommendations) where the periodic reset is not natural, and they mention the analysis of a learning algorithm in the \textit{infinite time horizon} as an open question.

We identify explicit conditions in Sec. \ref{sec:planning} that ensure the \textit{Bellman} equation of the entire Markov decision process (MDP). 
It is hard to analyze the vanilla Thompson sampling in this setting, and we adapt \textit{Thompson sampling with dynamic episodes} (TSDE) of \citet{ouyang2017learning} in the fully observable MDP. 
TSDE (Algorithm \ref{alg:TSDE}) has one deterministic and one random termination conditions and switches to a new episode if one of these is met. 
At the beginning of each episode, TSDE draws a system parameter using the posterior distribution from which it computes a policy and runs this policy throughout the episode. 
We theoretically prove a sub-linear regret bound of this algorithm and empirically test it on a simulated dynamic channel access problem.

\subsection{Main Result}
As mentioned earlier, our learner competes against the policy $\piStar = \mu(\thetaStar)$ without the knowledge of $\thetaStar \in \Theta$.
We denote the \textit{average long term reward} of $\piStar$ on the system $\thetaStar$ by $J_{\piStar}(\thetaStar)$, which is a well-defined notion under certain assumptions that will be discussed later.
Then we define the \textit{frequentist regret} by
\begin{equation}
\label{eq:freq.regret.def}
R(T;\thetaStar) 
= J_{\piStar}(\thetaStar)\cdot T - \E_{\thetaStar}\sumtt r_t,
\end{equation}
where $r_t$ is the learner's reward at time $t$. 
We focus on bounding the following \textit{Bayesian regret}
\begin{equation}
\label{eq:bayes.regret.def}
    BR(T) = \E_{\thetaStar \sim Q} R(T;\thetaStar),
\end{equation}
where $Q$ is a prior distribution over $\Theta$ and is known to the learner. 
Our main result is to bound the Bayesian regret of TSDE. 
\begin{theorem}
\label{thm:main}
The Bayesian regret of TSDE satisfies the following bound
\begin{align*}
BR(T) = \bigO(\sqrt{T} \log T),
\end{align*}
where the exact upper bound appears later in Sec. \ref{sec:regret}. 
\end{theorem}

\section{PRELIMINARIES}
We begin by formally defining our problem setting.

\subsection{Problem Setting}


As stated earlier, we focus on a Bayesian framework where the true system, denoted as $\thetaStar$, is a random object that is drawn from a prior distribution $Q$ before the interaction with the system begins. In line with~\citet{ouyang2017learning}, we assume that the prior is known to the learner, and we denote its support by $\Theta$.

At each time step $t$, the learner selects $N$ arms from $[K]$ which become \textit{active} while the others remain \textit{passive}. Following~\citet{ortner2012regret}, we impose the \textit{passive} Markov chains to be irreducible and aperiodic. As a result, we can associated with each arm $k$ the \textit{mixing time} of $\Ppassive_k$. Let $p^t_k(s)$ be the distributions of the state $s_k$ of arm $k$ starting from a state $s$ and remaining passive for $t$ steps, and let $p_k$ be the stationary distribution. Then, we define
\begin{equation}
    \Tmix_k (\epsilon)
    = 
    \inf \Big\{ t\geq 1 \text{ s.t.} \max_{s \in S_k}\|p^t_k(s) - p_k\|_1 \leq \epsilon \Big\}, 
    \label{eq:Tmix.k}
\end{equation}
and work under the assumption of known mixing time\footnote{The knowledge of $T^{mix}(\frac{1}{4})$ maybe relaxed to the knowledge of an upper bound of it, without affecting our result.}.
\begin{assumption}[Mixing times]\label{asm:mixing}
For all $k \in [K]$ and $\theta \in \Theta$, $\Ppassive_k$ is irreducible and aperiodic, and $\Tmix(\frac{1}{4}) := \max_{k, \theta} T^{mix}_k(\frac{1}{4})$ is known to the learner.
\end{assumption}

The learner's action at time $t $ is written as $A_t \in \{0,1\}^K$, $1$ indicating the active action. For all the chosen arms, the learner observes the state $s^t_k$ and receives a reward $r_k(s^t_k)$, where the rewards are deterministic \textit{known} functions of the state $r_k : S_k \rightarrow [0, 1]$ for all $k \in [K]$. The objective of the learner is to choose the best sequence of arms, given the history (state and actions) observed so far, which maximizes the long term average reward 
\begin{equation}
    \underset{t\rightarrow T}{\lim\sup} \frac{1}{T} \mathbb{E}\left( \sum_{t=1}^T \sum_{k:A_{t,k} = 1} r_k(s^t_k) \right).
    \label{eq:informal.partial.objective}
\end{equation}

\subsection{From POMDP to MDP}
\label{sec:POMDP}

By nature, the RMAB problem we consider is a \textit{partially observable Markov decision process} (POMDP) since the arms evolve in a Markovian fashion and we only observe the states of the active arms.
Nonetheless, one can turn this POMDP into a fully observable \textit{Markov decision process} (MDP) by introducing belief states, i.e., distributions over states given the history. 
Notice that the number of belief states become therefore (countably) infinite even if the original problem is finite. Following~\citet{ortner2012regret} and~\citet{jung2019regret}, we track the history introducing a \textit{meta-state} $\xi_t$, fully observed at time $t$, from which we can reconstruct the belief states. 
Formally, we define $\xi_t=(\xi_t^s,\xi_t^n)$ where 
\begin{align*}
\xi^s_t = (\sigma^t_1, \cdots, \sigma^t_K) 
\text{ and }
\xi^n_t = (n^t_1, \cdots, n^t_K).
\end{align*}
For each $k \in [K]$, $\sigma^t_k$ is the last observation of the state process $\{s^t_k\}_{t\geq 1}$ before time $t$, $n^t_k$ is time elapsed from this last observation. Further, it is clear that $\{\xi_t\}_{t\geq 1}$ is a Markov process on a countably infinite state space $S$. As a result, the maximization of the partially observable problem in Eq.~\ref{eq:informal.partial.objective} is equivalent to the maximization of the fully observable one
\begin{equation}
    \underset{t\rightarrow T}{\lim\sup} \frac{1}{T} \mathbb{E} \Big( \sum_{t=1}^T r_{\thetaStar}(\xi_t,A_t)\Big),
    \label{eq:informal.fully.objective}
\end{equation}
where 
\begin{equation}
r_{\theta}(\xi_t, A_t)
= \sum_{k: A_{t, k} = 1} \E_{\theta} [r_k (s^t_k) | \xi_t, A_t].
\label{eq:reward.fully.def}
\end{equation}
We use the notation $\E_\theta$ and $r_\theta$ to emphasize that the random behavior of $s^t_k$ is governed by the system $\theta$.
We also assume that the initial state $\xi_1$ is known to the learner.

\subsection{Policy Mapping}
\label{sec:policy.mapping}

To maximize the long term average reward in Eq.~\ref{eq:informal.fully.objective},~\citet{ortner2012regret} construct a finite approximation of the countable MDP which allows them, under a bounded diameter assumption, to compute $\epsilon$-optimal policy for a given $\theta$. However, their computational complexity is prohibitive for practical applications. As explained in the introduction, we follow a different approach, in line with~\citet{jung2019regret}, which achieves both tractability and optimality through the use of a policy mapping $\mu : \Theta \rightarrow \Pi$.
It associates each parameter $\theta$ with a stationary deterministic policy $\pi_\theta$.
To ensure the well-posedness of the long-term average reward, we impose the following assumption on $\mu$.
\begin{assumption}[Bounded span]\label{asm:bounded.span}
For all $\theta \in \Theta$, the parameter/policy pair $(\theta,\pi_{\theta})$ satisfies Cond.~\ref{cond:explicit.bellman}.
\end{assumption}
Cond.~\ref{cond:explicit.bellman} is formalized and discussed in detail in Sec.~\ref{sec:planning}. Asm.~\ref{asm:bounded.span} should be understood as the counterpart of the bounded diameter assumption made by~\citet{ortner2012regret} or the bounded span assumption by~\citet{ouyang2017learning} adapted to our policy mapping approach.

\section{ALGORITHM}
\label{sec:alg}

Algorithm \ref{alg:TSDE} builds on \textit{Thompson Sampling with Dynamic Episodes} (TSDE) of \citet{ouyang2017learning}. 
At the beginning of each episode $i$, we draw system parameters $\theta_i$ from the latest posterior $Q_{t_i}$, compute the policy $\pi_i = \mu(\theta_i)$, and run $\pi_i$ throughout the episode. 
We proceed to the next episode if one of the termination conditions, which will appear shortly, occurs.

\begin{algorithm}[ht]
	\begin{algorithmic}[1]
	    \STATE \textbf{Input} prior $Q$, policy mapping $\mu$,
	    \STATE \textcolor{white}{\textbf{Input}} mixing time $\Tmix$, initial state $\xi_1$
	    \STATE \textbf{Initialize} $Q_1 = Q$, $t = 1$, $t_0 = 1$
	    \FOR {episodes $i = 1, 2, \cdots$}
	    \STATE Set $t_i = t$ and $T_{i-1} = t_i - t_{i-1}$
	    \STATE Draw $\theta_i \sim Q_t$ and compute $\pi_i = \mu(\theta_i)$
	    \WHILE{not termination condition (Eq. \ref{eq:terminationCond})}
	    \STATE Select active arms $A_t = \pi_i(\xi_t)$
	    \STATE Observe states $s^t_k$ for active arms $k$
	    \STATE Update $\xi_t$ to $\xi_{t+1}$ and $Q_t$ to $Q_{t+1}$
	    \STATE Increment $t = t+1$
	    \ENDWHILE
	    \ENDFOR
	\end{algorithmic}
	\caption{TSDE in restless bandits}
	\label{alg:TSDE}
\end{algorithm}

Before introducing the termination conditions, let us discuss Asm. \ref{asm:mixing}.
As pointed out in~\citet[Eq. 1]{ortner2012regret},
we have 
$\Tmix_k(\epsilon) 
\leq 
\log_2(1/\epsilon) \Tmix(\frac{1}{4})$ for all $k \in [K]$ and $\epsilon > 0$.
As we want the accuracy of $\frac{1}{T}$, which will not affect the regret significantly, we define 
$$
\Tmix 
:= (\log_2 T)\Tmix(\frac{1}{4})
\ge \Tmix(\frac{1}{T}).
$$
Here we assume the time horizon $T$ is known. 
When it is unknown, 
we can use the \textit{doubling trick}
and get the same regret bound up to a constant factor. 
We remark that $\Tmix = \bigO(\log T)$.

For tuples $(k, s, n)$, we define
\begin{align*}
\tilde{N}_t(k, s, n) 
&= \sum_{\tau=1}^{t-1} 
    \ind (A_{\tau, k} = 1, \sigma^t_k = s, n^t_k = n).\\
\end{align*}


Then we introduce the truncated counter 
\begin{align*}
N_t(k, s, n) =
\begin{cases}
    \tilde{N}_t(k, s, n) &\text{ if } n < \Tmix \\
    \sum_{n^\prime \ge \Tmix}\tilde{N}_t(k, s, n^\prime) &\text{ if } n \geq \Tmix 
\end{cases}.
\end{align*}
The intuition behind this aggregation is that 
the distribution of the states remains similar for sufficiently large $n$, thanks to the mixing time. 
As a result, the possible number of tuples $(k, s, n)$ with $n \leq \Tmix $ is at most $ \sum_k |S_k| \cdot \Tmix$.
When there is no ambiguity, 
we write 
$(k, s, n) = \zeta$
for brevity and let $Z$ be the set of all possible values of $\zeta$.

We \textit{terminate} the episode $i$ if 
\begin{equation}
\label{eq:terminationCond}
t > t_i + T_{i-1}
\text{ or }
N_t (\zeta) > 2 N_{t_k}(\zeta)
\text{ for some } 
\zeta \in Z,
\end{equation}
where $T_i$ represents the length of episode $i$.
This quantity can differ for each episode. 
This is where the name dynamic episodes comes from. 
In addition, the second condition makes the quantity $T_i$ random, and one recovers the well-known \textit{lazy update scheme} from this condition~\citep{jaksch2010near,ouyang2017learning}. The underlying intuition is that one should update the policy only after gathering enough additional information over the unknown Markov process. 

\section{PLANNING PROBLEM}
\label{sec:planning}


The MDP reformulation in Sec.~\ref{sec:POMDP} reduces the objective to maximizing Eq.~\ref{eq:informal.fully.objective}. However, we inherit from the original POMDP problem severe difficulties in the \textit{planning} task. 
For example, given the parametrization $\thetaStar$, how to efficiently compute a \textit{stationary} and \textit{deterministic} policy $\pi$ (i.e., $\pi$ maps a state $\xi$ to an  action $A$ in a deterministic manner) that maximizes the average long term reward, and more importantly, does such policy exist? Unfortunately, the average reward POMDP problem is not well understood in contrast with the finite state average reward MDP. In particular, it is known~\citep{bertsekas1995dynamic} that the long term average reward may not be constant w.r.t.\ the initial state. 
Even when this holds,
\textbf{1)} The Bellman equation may not have a solution. 
\textbf{2)} Value Iteration may fail to converge to the optimal average reward.
\textbf{3)} There may not exist an optimal policy, stationary or non-stationary. 
\textbf{4)} Finally, even when the optimal policy exists,~\citet{papadimitriou1999complexity} show that it is generally PSPACE hard to compute it.

To overcome this difficulty,~\citet{ortner2012regret} perform a state aggregation to reduce the countably infinite MDP into a finite one, which under the bounded diameter assumption can be solved using standard techniques. Although this reduction allows them to compute an $\epsilon$-optimal policy, the computational complexity of their approach remains prohibitive for practical application. 
On the other hand, a significant amount of work has been done to design \textit{good} policies in the RMAB framework, for instance the best fixed arm policy (that is optimal in the classical MAB framework), the myopic policy~\citep{javidi2008optimality}, or the Whittle index policy~\citep{whittle1988restless,liu2010indexability}. In line with~\citet{jung2019regret}, we leverage this prior knowledge following an alternative approach that consists in competing with the best policy within some known class of policies. 
Formally, let $\Pi$ be the set of stationary deterministic policies, and we assume a policy mapping $\mu :\Theta \rightarrow \Pi$ is given and known to the learner. 
This set of deterministic mappings is quite rich in that the optimal policy can be also represented when it exists. If one cares more about the efficiency, one can use some efficient mappings while there is a trade-off of weakening the competitor.

Finally, in contrast to~\citet{ortner2012regret}, our approach does not turn the countable MDP problem into a finite one. Hence, it requires a further condition on the parameter space $\Theta$ and the policy mapping $\mu$ for the average reward criterion in Eq.~\ref{eq:informal.fully.objective} to be well-posed. More precisely, we expect the average reward to be independent of the initial state and associated to a Bellman equation, with a bias function of a bounded span. For a given $\theta \in \Theta$ and associated policy $\pi_\theta = \mu(\theta)$, we introduce the following conditions.
\begin{condition}
\label{cond:implicit.bellman}
Let $\mathcal{V}$ be the set of bounded span real-valued function. There exists $v \in \mathcal{V}$ and a constant $g$ which satisfy for all $\xi \in S$,
\begin{equation*}
    g + v(\xi) = r_\theta(\xi,\pi_\theta(\xi)) + \E_{\theta}[ 
        v(\xi^\prime) |\xi, \pi_\theta(\xi) 
    ],
\end{equation*}
where the expectation is taken over $\xi^\prime$ evolving from $\xi$ given the action $\pi_\theta(\xi) $ and the system $\theta$.
\end{condition}
Under Cond.~\ref{cond:implicit.bellman}, it is known (see Prop.~\ref{prop:cond.1.2.bellman}) that the long term average reward of $\pi_\theta$ is well-defined (the $\lim\sup$ reduces to the standard $\lim$), independent of the initial state $\xi_1$, and associated with the Bellman equation with a bounded span bias function. 
However, Cond.~\ref{cond:implicit.bellman} is implicit and uneasy to assert as it relies on the existence result\footnote{If a function $v$ satisfies Cond.~\ref{cond:implicit.bellman}, it is not unique since adding any constant to $v$ still meet the requirement.}. This motivates the alternative condition, known as the \textit{discounted approach} in the literature. 
\begin{condition}
For any $\beta \in (0,1)$, let $v_{\pi_\theta}^\beta$ be the discounted infinite horizon value function defined as
\begin{equation*}
    v^\beta_{\pi_\theta}(\xi) = \mathbb{E}_\theta\left(\sum_{t=1}^{\infty} \beta^t r_\theta(\xi_t,\pi_\theta(\xi_t)) | \xi_1 = \xi \right).
\end{equation*}
Then $\sup_{(\xi,\xi^\prime)\in S^2} v_{\pi_\theta}^\beta(\xi) - v_{\pi_\theta}^\beta(\xi^\prime)$ is uniformly bounded for all $\beta \in (0,1)$.
\label{cond:explicit.bellman}
\end{condition}
The introduction of the discount factor $\beta \in (0,1)$ guarantees that $v^\beta_{\pi_\theta}$ is a well-defined function, and hence Cond.~\ref{cond:explicit.bellman} is reduced to assert the uniform boundedness of a known family of function. Further, it also guarantees that the long term average reward is well-defined as it implies Cond.~\ref{cond:implicit.bellman}.

\begin{proposition}
\label{prop:cond.1.2.bellman}
Let $\theta \in \Theta$ be a system parameter and $\pi_\theta = \mu(\theta)$ be a policy. Then the followings hold. 
\begin{itemize}[leftmargin=*]
    \item Cond.~\ref{cond:explicit.bellman} implies Cond.~\ref{cond:implicit.bellman}.
    \item Under Cond.~\ref{cond:implicit.bellman} (or Cond.~\ref{cond:explicit.bellman}), the quantity
    \begin{equation}
       \! J_{\pi_\theta}(\theta)\! =\! \lim_{T \rightarrow \infty} \frac{1}{T} \E_\theta \left( \sum_{t=1}^T r_\theta(\xi_t,\pi_{\theta}(\xi_t)) |\xi_1 = \xi\right)
        \label{eq:averageReward}
    \end{equation}
     is constant and independent of the initial state. 
     Further, there exists a non-negative function $h_\theta$, with bounded span $C_\theta = \sup_{(\xi,\xi^\prime)\in S^2} h_\theta(\xi) - h_\theta(\xi^\prime) < \infty$, such that for any $\xi \in S$,
     \begin{equation}
         \! J_{\pi_\theta}(\theta) + h_\theta(\xi) = r_\theta(\xi,\pi_\theta(\xi)) + \E_{\theta}[
            h_\theta(\xi^\prime) |\xi,\pi_\theta].
        \label{eq:bellman}
     \end{equation}
\end{itemize}
We denote as $H = \sup_{\theta \in \Theta} C_\theta$ the uniform upper bound on the span.
\end{proposition}
The proof of Prop.~\ref{prop:cond.1.2.bellman} can be adapted from \citet[Thm.8.10.7]{puterman2014markov} for a given (i.e., not necessarily optimal) policy. We postpone the proof to App.~\ref{app:proof.prop.bellman}.




\section{REGRET BOUND}
\label{sec:regret}
In this section, we bound the Bayesian regret of TSDE (Algorithm \ref{alg:TSDE}). The analysis crucially relies on four distinct properties: 
\textbf{1)} the Bellman equation in Eq.~\ref{eq:bellman} satisfied by the average cost at each policy update,
\textbf{2)} the Thompson sampling algorithm which samples parameters $\theta_{i}$ according to the posterior, hence ensuring that $\thetaStar$ and $\theta_{i}$ are conditionally identical in distribution, 
\textbf{3)} the concentration of the empirical estimates around the $\thetaStar$, and
\textbf{4)} the update scheme in Eq. \ref{eq:terminationCond} which controls the number of episodes while preserving sufficient measurability of the termination times.

We provide here a proof sketch to explain how we leverage those properties and how they translate in key intermediate results that allow us to obtain the final bound. The formal proofs can be found in App.~\ref{app:regret.proofs}.

\subsection{Regret Decomposition}
\label{section:regretDecomposition}

Under Asm.~\ref{asm:bounded.span}, Prop.~\ref{prop:cond.1.2.bellman} ensures that each sampled parameter policy pair $(\theta_i,\pi_i)$ satisfies the Bellman equation (Eq.~\ref{eq:bellman}):
\begin{align*}
r_{\theta_i}(\xi, \pi_i(\xi))
= J_{\pi_i}(\theta_i) + v_{\theta_i}(\xi)
-\E_{\theta_i} [v_{\theta_i}(\xi^\prime) | \pi_i, \xi].
\end{align*}
As a result, we can decompose on each episode $i$ the frequentist regret and obtain over $T$, 
\begin{align*}
R(T;\thetaStar) 
&= J_{\piStar}(\thetaStar)\cdot T - \E_{\thetaStar}\sum_{i=1}^{M_T}\sum_{t=t_i}^{t_{i+1} - 1} r_{\thetaStar}(\xi_t, A_t) \\
&=: R_0 + R_1 + R_2 + R_3,
\end{align*}
where
\begin{align*}
R_0 
&= J_{\piStar}(\thetaStar)\cdot T 
- \E_{\thetaStar}\sum_{i=1}^{M_T}  J_{\pi_i}(\theta_i) \cdot T_i\\
R_1
&= \E_{\thetaStar}\sum_{i=1}^{M_T}\sum_{t=t_i}^{t_{i+1} - 1} 
v_{\theta_i}(\xi_{t+1}) - v_{\theta_i}(\xi_t) \\
R_2
&=\E_{\thetaStar}\sum_{i=1}^{M_T}\sum_{t=t_i}^{t_{i+1} - 1} 
\E_{\theta_i} [v_{\theta_i}(\xi^\prime) | \pi_i, \xi_t]
 - v_{\theta_i}(\xi_{t+1}) \\
R_3
&= \E_{\thetaStar}\sum_{i=1}^{M_T}\sum_{t=t_i}^{t_{i+1} - 1}
(r_{\theta_i} - r_{\thetaStar})(\xi_t, \pi_i(\xi_t)).
\end{align*}
See App.~\ref{app:regret.proofs} for a more detailed derivation.

\textbf{Bounding $R_0$.} The first regret term is addressed thanks to the well-known expectation identity (see~\cite{russo2014learning}), leveraging that conditionally, $\thetaStar \overset{d}{=} \theta_i$. 
\begin{lemma}[Expectation identity]
\label{lemma:expectationIdentity}
Suppose $\thetaStar$ and $\theta_i$ have the same distribution given a history $\Hc$. 
For any $\Hc$-measurable function $f$, we have
\[
\E[f(\thetaStar)|\Hc]
=
\E[f(\theta_i)|\Hc].
\]
\end{lemma}
As pointed out in~\citet{ouyang2017learning}, one cannot apply Lemma~\ref{lemma:expectationIdentity} directly to $J_{\pi_i}(\theta_i)$ and $J_{\piStar}(\thetaStar)$ because of the measurability issue arising from the lazy-update scheme in Eq.~\ref{eq:terminationCond}. In line with~\citet{ouyang2017learning}, we overcome this difficulty thanks to the first deterministic termination rule in Eq.~\ref{eq:terminationCond}. Taking the expectation w.r.t. $\thetaStar$ leads to the following lemma.
\begin{lemma}[\citet{ouyang2017learning}, Lemma~3 and 4]
\label{lemma:r0}
\begin{equation*}
    \E_{\thetaStar \sim Q} R_0 \le N \cdot \E_{\thetaStar \sim Q} M_T,
\end{equation*}
where $M_T$ is the total number of episodes until time $T$.
\end{lemma}

\textbf{Bounding $R_1$.}
Clearly, $R_1$ involves telescopic sums over each episode $i$. As a result, it solely depends on the number of policy switches and on the uniform span bound $H$ in Prop.~\ref{prop:cond.1.2.bellman}.
\begin{lemma}
\label{lemma:r1}
\begin{equation*}
R_1 \le H \cdot \E M_T.
\end{equation*}
\end{lemma}
As a result, both $R_0$ and $R_1$ reduce to a fine bound over the number of episodes, $M_T$.

\textbf{Bounding $R_2$ and $R_3$.} Finally, the last regret terms are dealing with the model misspecification.
That is to say, they depend on the \textit{on-policy} error between the empirical estimate and the true transition model. Formally, Lemma~\ref{lemma:r2} and~\ref{lemma:r3} show that they scale with 
\begin{equation*}
    \Delta_T = \sum_{i=1}^{M_T}\sum_{t=t_i}^{t_{i+1} - 1}\sum_{\text{active arms } k}{
    ||
        (\hat{p}_{t_i} - p_{\thetaStar})(k, \sigma^t_k, n^t_k)
    ||_1,
}
\end{equation*}
where $p_{\theta}(\cdot; k, \sigma, n)$ is the probability distribution of arm $k$'s state under parametrization $\theta$ and $\hat{p}_{t_i}$ is its empirical estimate at the beginning of episode $i$.
The core of the proofs thus lies in deriving a high-probability confidence set whose associated on-policy error $\Delta_T$ is cumulatively bounded by $\sqrt{T}$. 
We state the lemmas here and postpone the proofs to App.~\ref{app:regret.proofs}.
\begin{lemma}
\label{lemma:r2}
$R_2$ satisfies the following bound
\[
R_2 \le 28H \sum_{k=1}^K |S_k|
\sqrt{
    N\Tmix T \log (\Tmix T)
}.
\]
\end{lemma}
\begin{lemma}
\label{lemma:r3}
$R_3$ satisfies the following bound
\[
R_3 \le 28 \sum_{k=1}^K |S_k|
\sqrt{
    N\Tmix T \log (\Tmix T)
}.
\]
\end{lemma}

We detail the construction and probabilistic argument of the confidence set later in the section.

\subsection{Bounding the Number of Episodes}
As breifly discussed in Sec. \ref{sec:alg}, 
each episode has a random length $T_i$, 
and the number of episodes $M_T$ also becomes random. 
In order to bound $R_0$ and $R_1$, we first bound this quantity. As discussed in \citet{osband2014near},
the specific structure of our problem due to the MDP formulation of the original POMDP problem allows us to guarantee a tighter bound w.r.t. the number of states than straightforwardly applying the TSDE analysis on the meta-state $\xi$. 
In particular, we leverage this structure to obtain a bound that depends on the number of states through the summation $\sum_{k=1}^K |S_k|$ instead of the product $\prod_{k=1}^K |S_k|$. 
\begin{lemma}
\label{lemma:numEp}
The number of episodes $M_T$ satisfies the following inequality almost surely
\[
M_T 
\le 2 \sqrt{
    (\sum_{k=1}^K |S_k|) \Tmix T\log NT
}.
\]
\end{lemma}
\begin{proof}
Following \citet{ouyang2017learning}, 
we define macro episodes with start times $t_{n_i}$
for a sub-sequence 
$\{n_i\} \subset [M_T]$
such that 
$n_1 = 1$ and
\[
t_{n_{i+1}}
=
\min\{
    t_k > t_{n_i} 
    |
    N_{t_k} (\zeta) > 2 N_{t_{k-1}}(\zeta)
    \text{ for some }
    \zeta
\}.
\]
Note that the macro episode starts when the second termination criterion happens. 
\citet{ouyang2017learning} prove in their Lemma 1 that
\begin{equation}
\label{eq:numEp1}
M_T 
\le 
\sqrt{2MT},
\end{equation}
where $M$ is the number of macro episodes. 
We claim
\begin{equation}
\label{eq:numEp2}
M 
\le 
2(\sum_{k=1}^{K} |S_k|) \Tmix \log NT,
\end{equation}
which prove our lemma when combined with Eq. \ref{eq:numEp1}.

For each $\zeta = (k, s, n) \in Z$, we define 
\[
M(\zeta) 
= 
|\{
    i \le M_T 
    | 
    N_{t_i}(\zeta) > 2 N_{t_{i-1}}(\zeta)
\}|.
\]

This means that $N_t(\zeta)$ gets doubled $M(\zeta)$ times 
out of $M_T$ episodes. 
It leads to the following inequality
\[
2^{M(\zeta)} \le N_{T+1} (\zeta)
\text{, or }
M(\zeta) \le 2 \log N_{T+1}(\zeta).
\]

Then we have 
\begin{align*}
M 
&\le 1 + \sum_{\zeta \in Z}M(\zeta)\\
&\le 1 + 2\sum_{\zeta \in Z} \log N_{T+1}(\zeta)\\
&\le 1 + 2 (\sum_{k=1}^K |S_k |) \Tmix \log 
\frac{
    \sum_{\zeta}N_{T+1}(\zeta)
}{
    (\sum_k |S_k |) \Tmix
}\\
&= 1 + 2 (\sum_{k=1}^K |S_k |) \Tmix \log 
\frac{
    NT
}{
    (\sum_k |S_k |) \Tmix
} \\
&\le 2 (\sum_{k=1}^K |S_k |) \Tmix \log NT,
\end{align*}
where we added $1$ to account for the initial case $n_1 = 1$
and the third inequality holds due to Jensen's inequality 
along with the fact that 
$|Z| \le \sum_k |S_k |\cdot \Tmix$. 
The equality holds because 
$\sum_{\zeta}N_{T+1}(\zeta)$
is the total number of active arms until time $T$.
This proves our claim (Eq. \ref{eq:numEp2}) and therefore the lemma.
\end{proof}

\subsection{Confidence Set}
To bound $R_2$ and $R_3$, 
we construct a confidence set for the system parameters $\theta$. 
Recall that $\zeta$ represents $(k, s, n)$. 
Suppose at time $t$, the state of arm $k$ was observed to be $s$ in $n$ rounds ago.
Let $p_\theta(\zeta)$ denote the probability distribution of 
the arm's state if the true system were $\theta$.
For an individual probability weight, 
we write 
$p_\theta(s^\prime;\zeta) = p_\theta(s^\prime;k, s, n)$ 
for $s^\prime \in S_k$.
Using the $N_t(\zeta)$ samples collected so far, 
we can also compute an empirical distribution $\phat_t(\zeta)$. 
We construct a confidence set as a collection of $\theta$
such that $p_\theta(\zeta)$ is close to $\phat_t(\zeta)$.
Namely in episode $i$, we define 
\begin{align*}
\Theta_i
=
\{
    \theta \in \Theta
    |
    \forall \zeta \in Z,
    ||(p_\theta - \phat_{t_i})(\zeta)||_1
    \le
    c_i(\zeta)
\},
\end{align*}
where 
$
c_i(\zeta) = 
c_i(k, s, n) =
\sqrt{
    \frac{
        8 |S_k| \log 1/\delta
    }{
        1 \vee N_{t_i}(\zeta)
    }
}.
$

Since $\Theta_i$ is $\Hc_{t_i}$-measurable, 
Lemma \ref{lemma:expectationIdentity} provides 
\begin{align*}
\prob(\thetaStar \notin \Theta_i | \Hc_{t_i})
=
\prob(\theta_i \notin \Theta_i | \Hc_{t_i}).
\end{align*}
The following lemma bounds this probability. 

\begin{lemma}
\label{lemma:confidenceProb}
For every episode $i$, we can bound
\[
\prob(\thetaStar \notin \Theta_i | \Hc_{t_i})
=
\prob(\theta_i \notin \Theta_i | \Hc_{t_i})
\le
\sum_{k=1}^K|S_k|\cdot \delta\Tmix.
\]
\end{lemma}
\begin{proof}
For an episode $i$, pick $(k, s, n) = \zeta \in Z$ and let $m = N_{t_i}(\zeta)$. 
If $m$ equals to $0$, then $c_i(\zeta) > 1$ and the inequality 
$
||(p_\theta - \phat)(\zeta)||_1
\le
c_i(\zeta)
$
becomes trivial. 
Suppose $m > 0$.
We first analyze the case $n < \Tmix$.
\citet{weissman2003inequalities} show that 
\begin{equation}
\label{eq:confidence1}    
\prob(||(p_\theta - \phat)(\zeta)||_1 \ge \epsilon)
\le
2^{|S_k|} \exp(
    -\frac{
        m\epsilon^2
    }{
        2
    }
).
\end{equation}
Setting 
$
\epsilon 
= c_i(\zeta)
= \sqrt{
    \frac{
        8 |S_k| \log 1/\delta
    }{
        n
    }
}
$
, we get 
\begin{equation}
\label{eq:confidence2}
\prob(||(p_\theta - \phat)(\zeta)||_1 \ge c_i(\zeta))
\le
\delta.
\end{equation}

For the case $n = \Tmix$, we want to prove the same probability bound in Eq. \ref{eq:confidence2} but cannot directly use Eq. \ref{eq:confidence1} due to aggregation.
We can still show a similar bound by using the proof technique by \citet{weissman2003inequalities}. 

For simplicity, write $p_\theta(\zeta) = p$, $\phat(\zeta) = \phat$, and $c_i(\zeta) = c$.
Then it can be easily checked that 
\begin{equation*}
||p - \phat||_1
=
2 \max_{A \subset S_k} p(A) - \phat(A).
\end{equation*}
Using this and the union bound, we can write 
\begin{align}
\label{eq:confidence3}
\prob(||p - \phat||_1 \ge c)
\le 
\sum_{A \subset S_k} 
\prob(p(A) - \phat(A) \ge \frac{c}{2}).
\end{align}
By the definition of $\Tmix$, we have 
\[
|p(A) - \E \phat(A)| 
<
\frac{1}{T} 
<
\frac{c}{4}.
\]
Then Hoeffding's inequality implies that
\begin{align*}
\prob(p(A) - \phat(A) \ge \frac{c}{2})
&\le
\prob(\E \phat(A) - \phat(A) \ge \frac{c}{4})\\
&\le
\exp(-\frac{mc^2}{8}).
\end{align*}
Plugging this in Eq. \ref{eq:confidence3}, we get 
\[
\prob(||p - \phat||_1 \ge c)
\le 
2^{|S_k|}\exp(-\frac{mc^2}{8})
\le \delta,
\]
which shows Eq. \ref{eq:confidence2} for the case $n = \Tmix$.

Since $|Z| \le \sum_{k=1}^K|S_k|\cdot \Tmix$, 
applying the union bound finishes the proof.
\end{proof}
Furthermore, the confidence set satisfies that the cumulative on-policy error $\Delta_T$ (see Sec. \ref{section:regretDecomposition}) is bounded.
\begin{lemma}
\label{lemma:onpolicyerror}
On the high-probability event $\thetaStar \in \cap_{i\leq M_T} \Theta_i$, we can show
\begin{equation*}
\Delta_T \leq 12
\sqrt{
    N\Tmix T \log 1/\delta
}
\sum_{k=1}^K |S_k|.
\end{equation*}
\end{lemma}
The proof of Lemma~\ref{lemma:onpolicyerror} is postponed to App.~\ref{app:regret.proofs}.  We want to emphasize that the set $\Theta_i$ only appears in the proof
and it has nothing to do with running TSDE. 
For example, we can set an arbitrary value for $\delta$ 
to make the proof works. 
The main idea of bounding $R_2$ and $R_3$ is that 
the event $\thetaStar, \theta_i \in \Theta_i$ happens 
with high probability, and if so, 
then $\piStar$ and $\pi_i$ behave similarly.

\subsection{Putting Everything Together}
Plugging Lemma \ref{lemma:r0}, \ref{lemma:r1}, \ref{lemma:r2}, and \ref{lemma:r3} into the regret decomposition, we prove our main result.

\newtheorem*{new.thm:main}{Theorem~\ref{thm:main}}
\begin{new.thm:main}[Exact regret bound, restated]
The Bayesian regret of TSDE is bounded by
\begin{align*}
2(H&+N)\sqrt{
    (\sum_{k=1}^K |S_k|) \Tmix T\log NT
}\\
&+
28(H+1)(\sum_{k=1}^K |S_k|)
\sqrt{
    N\Tmix T \log (\Tmix T)
},
\end{align*}
where $\Tmix = (\log_2 T)\Tmix(\frac{1}{4}) = \bigO(\log T)$.
\end{new.thm:main}

\section{EXPERIMENTS}
We empirically evaluated TSDE (Algorithm \ref{alg:TSDE}) on simulated data. 
Following \citet{jung2019regret}, 
we chose the Gilbert-Elliott channel model in  Figure \ref{fig:GEmodel} to model each arm. 
This model assumes binary states and is widely used in communication systems (e.g., see \citet{liu2010indexability}).

\begin{figure}[ht]
\label{sec:experiments}
\centering
\includegraphics[width=\columnwidth]{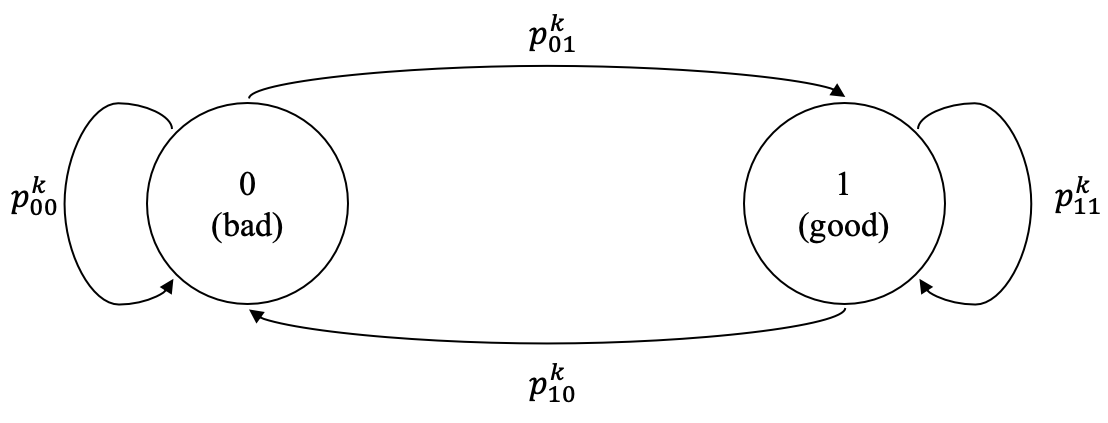}
\caption{The Gilbert-Elliott channel model}
\label{fig:GEmodel}
\end{figure}

For simplicity, we assumed $\Pactive = \Ppassive$ and $r_k(s) = s$. 
This means that the learner's action does not affect the transition matrix and the binary reward equals one if and only if the state is good. 
We also assumed the initial states of the arms are all good. 
Each arm has two parameters: $p^k_{01}$ and $p^k_{11}$. 
We set the prior to be uniform over a finite set $\{.1, .2, \cdots, .9\}$. 
Expectations are approximated by the Monte Carlo simulation with size $100$ or greater. 

We investigated three \textit{index-based policies}: the best fixed arm policy, the myopic policy, and the Whittle index policy. 
Index-based policies compute an index for each arm only using the samples from this arm and choose the top $N$ arms. 
Due to their decoupling nature, 
these policies are computationally efficient. 
The best fixed arm policy computes the expected reward according to the stationary distribution.
The myopic policy maximizes the expected regret of the current round. 
The Whittle index policy is first introduced by \citet{whittle1988restless} and shown to be powerful in this particular setting by \citet{liu2010indexability}.
The Whittle index policy is very popular in RMABs as it can efficiently approximate the optimal policy in many different settings. 
As a remark, all these policies are reduced to the best fixed arm policy in the stationary bandits. 

\begin{figure}[ht]
\vskip 0.2in
\begin{center}
\centerline{\includegraphics[width=\columnwidth]{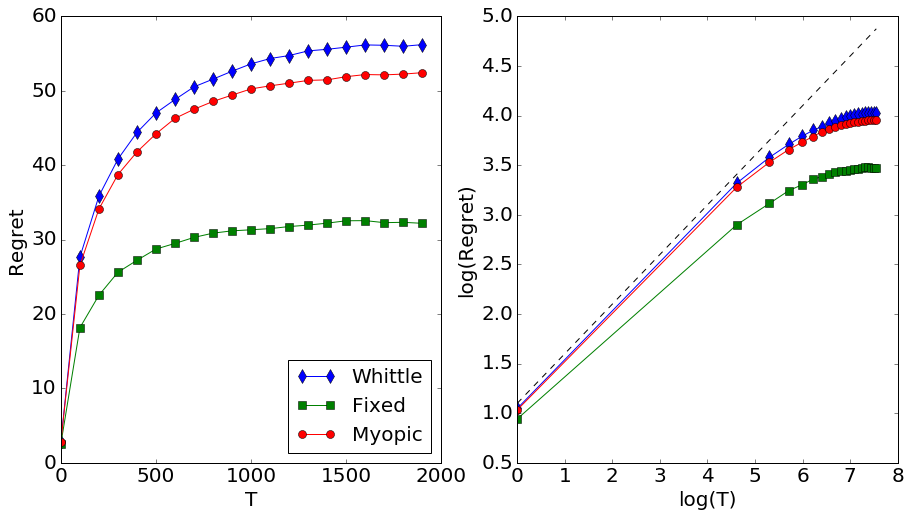}}
\caption{Bayesian regrets of TSDE (left) and their $\log$-$\log$ plots (right)}
\label{fig:bayesian}
\end{center}
\end{figure}

We first analyzed the Bayesian regret. 
Here we used $T = 2000$, $K = 8$, and $N = 3$. 
The true system $\thetaStar$ was actually drawn from the uniform prior. 
The average rewards smoothed by the prior, $\E_{\thetaStar \sim Q}J_{\piStar}(\thetaStar)$,
were $2.05$ (fixed), $2.16$ (myopic), and $2.17$ (Whittle), 
showing the power of the Whittle index policy. 
As described in Figure \ref{fig:bayesian}, 
the Bayesian regrets were sub-linear regardless of the competitor policy.
The log-log plot shows that they are indeed $\tilde{\bigO}(\sqrt{T})$ as the dotted line has a slope of $0.5$. 

\begin{figure}[ht]
\vskip 0.2in
\begin{center}
\centerline{\includegraphics[width=\columnwidth]{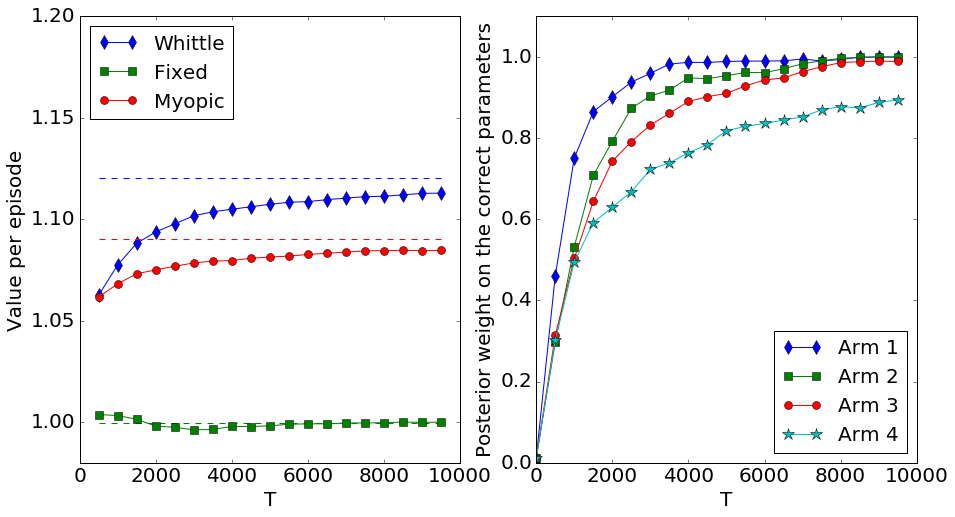}}
\caption{Average rewards of TSDE converge to their benchmarks (left); Posterior weights of the true parameters monotonically increase to one (right)}
\label{fig:freq}
\end{center}
\end{figure}

Then we tested the frequentist setting to empirically validate that TSDE still performs well in this setting even though our theory only bounds the Bayesian regret. 
We chose $T=10000$, $K=4$, $N=2$, and 
\[\{(p^{k}_{01}, p^{k}_{11})\}_{k=1, 2, 3, 4} = \{(.3, .7), (.4, .6), (.5, .5), (.6, .4)\}.
\]
We again adopted the setting from \citet{jung2019regret}. 
This $\thetaStar$ is particularly interesting because each arm has the same stationary distribution of $(.5, .5)$. 
This means that the best fixed arm policy becomes indifferent among the arms. 
The average rewards, $J_{\piStar}(\thetaStar)$,
were $1.00$ (fixed), $1.09$ (myopic), and $1.12$ (Whittle), again justifying the power of Whittle index policy. 
On the left plot of Figure \ref{fig:freq}, 
three horizontal dotted lines represent $J_{\piStar}(\thetaStar)$ for each of the competitors. 
The solid lines show the time-averaged cumulative rewards, $\frac{1}{t}\E_{\thetaStar}\sum_{\tau=1}^t r_{\thetaStar}(\xi_\tau, A_\tau)$.
Every solid line converged to the dotted line. 
The right figure plots the posterior probability of the true parameters using the Whittle index policy. 
For all arms, these probabilities monotonically increased to one, illustrating that TSDE were learning $\thetaStar$ properly. 
From this, we can assert that TSDE still performs reasonably well at least when the true parameters lie on the support of the prior. 

\subsubsection*{Acknowledgements}
AT and YJ acknowledge the support of NSF CAREER grant IIS-1452099. AT was also supported by a Sloan Research Fellowship.

\bibliography{ts_rmab} 

\clearpage
\newpage
\onecolumn
\begin{appendices}
\section{Proof of Prop.~\ref{prop:cond.1.2.bellman}}
\label{app:proof.prop.bellman}

We first prove that Cond.~\ref{cond:implicit.bellman} guarantees the constant average cost and the associated Bellman equation and then show that Cond.~\ref{cond:explicit.bellman} implies Cond.~\ref{cond:implicit.bellman}.
\begin{enumerate}
    \item Let $\theta \in \Theta$ and $\pi_\theta = \mu(\theta)$ satisfy Cond.~\ref{cond:implicit.bellman} for some bounded function $v \in \mathcal{V}$ and constant $g$. Then, for all $\xi \in S$,
    \begin{equation*}
        r_\theta(\xi,\pi_\theta(\xi)) = g + v(\xi) - \mathbb{E}_{\theta}[v(\xi^\prime) | \xi],
    \end{equation*}
    where the next "meta"-state $\xi^\prime \sim \prob_\theta( \cdot |\xi, A = \pi_\theta(\xi))$ is drawn according to the Markov transition probability of $\{\xi_t\}_{t\geq 1}$ knowing the current state $\xi$ and the action $A = \pi_\theta(\xi)$ under parametrization $\theta$. Thus,
    \begin{equation*}
    \begin{aligned}
        \sum_{t=1}^T r_\theta(\xi_t,\pi_\theta(\xi_t)) &= T g + \sum_{t=1}^T v(\xi_t) - \E_\theta[
            v(\xi_{t+1}) |\xi_t, A_t = \pi_\theta(\xi_t)
        ]\\
        &= T g + \sum_{t=1}^T v(\xi_{t+1}) - \E_\theta[
            v(\xi_{t+1}) |\xi_t, A_t = \pi_\theta(\xi_t)
        ]
        + v(\xi_1) - v(\xi_{T+1}).
    \end{aligned}
    \end{equation*}
    Multiplying by $\frac{1}{T}$ both sides of the equation and taking the expectation given $\xi_1$ leads to
    \begin{equation*}
        \frac{1}{T} \mathbb{E}_\theta \left( \sum_{t=1}^T r_\theta(\xi_t,\pi_\theta(\xi_t)) | \xi_1\right) = g +  \frac{1}{T} \mathbb{E}_\theta \Big( v(\xi_1) - v(\xi_{T+1}) | \xi_1 \Big).
    \end{equation*}
    Finally, since $v$ is bounded, letting $T \rightarrow \infty$ one has $\frac{1}{T} \mathbb{E}_\theta \Big( v(\xi_1) - v(\xi_{T+1}) | \xi_1 \Big) \rightarrow 0$ and thus 
    \begin{equation*}
        J_{\pi_\theta}(\theta) := \lim_{T \rightarrow \infty} \frac{1}{T} \mathbb{E}_\theta \left( \sum_{t=1}^T r_\theta(\xi_t,\pi_\theta(\xi_t)) | \xi_1\right) = g.
    \end{equation*}
    Futhermore, since $g$ is constant, it ensures that $J_{\pi_\theta}(\theta)$ is independent of the initial state. Replacing $g$ by $J_{\pi_\theta}(\theta)$ in Cond.~\ref{cond:implicit.bellman}, we directly obtain that $J$ is associated with the Bellman equation. 
    Since the function $v$ is arbitrary up to constant term (it still satisfies the Bellman equation and does not affect the span), we can set it without loss of generality to be non-negative defining $h_\theta(\xi) = v(\xi) - \inf_{\xi} v(\xi)$ and
    the pair $(J_{\pi_\theta}(\theta), h_\theta)$ satisfies the Bellman equation (Eq. \ref{eq:bellman}). 
    Additionally, we have $$C_\theta = \sup_{(\xi,\xi^\prime)\in S^2} h_\theta(\xi) - h_\theta(\xi^\prime) = \sup_{(\xi,\xi^\prime)\in S^2} v(\xi) - v(\xi^\prime)< \infty.$$

    \item We now show that Cond.~\ref{cond:explicit.bellman} implies Cond.~\ref{cond:implicit.bellman}. 
    The proof is adapted from~\citet[Thm.8.10.7]{puterman2014markov} which is derived for optimal policies. The core idea is to consider a sequence of discount factor $\beta_n \rightarrow 1$ and to choose an appropriate subsequence (also indexed by $n$ for ease of notation) to assert the existence of $g$ and $v \in \mathcal{V}$ thanks to the uniform boundedness of $|v^{\beta}_{\pi_\theta}|$.\\
    First, notice that for all $\xi \in S$, $r_\theta(\xi,\pi_\theta(\xi)) \in [0,N]$ and thus that $v^\beta_{\pi_\theta}(\xi) \in [0, \frac{N}{1 - \beta}]$ for all $\beta \in (0,1)$. 
    Also, it is well known that $v^\beta_{\pi_\theta}(\xi)$ satisfies the discounted Bellman equation:
    \begin{equation*}
        v^\beta_{\pi_\theta}= \mathcal{T}_\beta( v^\beta_{\pi_\theta}), \quad \text{ where } \mathcal{T}_\beta( v^\beta_{\pi_\theta} )(\xi) = r_\theta(\xi,\pi_\theta(\xi)) + \beta \mathbb{E}_\theta\big( v^\beta_{\pi_\theta}(\xi^\prime) | \xi \big).
    \end{equation*}
    Let $\bar{\xi} \in S$ be an arbitrary state and define $\bar{v}^\beta(\xi) = v^\beta_{\pi_\theta}(\xi) - v^\beta_{\pi_\theta}(\bar{\xi})$. Clearly, $\bar{v}^\beta$ is uniformly bounded and $\bar{v}^\beta$ satisfies
    \begin{equation}
        \bar{v}^\beta + (1- \beta) v^\beta_{\pi_\theta}(\bar{\xi}) = \mathcal{T}_\beta(\bar{v}^\beta).
        \label{eq:proof.bellman.1}
    \end{equation}
    Since $\bar{v}^\beta$ and $r_\theta$ are uniformly bounded, so is $(1- \beta) v^\beta_{\pi_\theta}(\bar{\xi})$. Further, the Bolzano-Weierstrass theorem for bounded sequence together with a standard diagonal argument ensures that there exists a subsequence $\beta_n \rightarrow 1$ such that 
    \begin{itemize}
        \item $(1- \beta_n) v^{\beta_n}_{\pi_\theta}(\bar{\xi}) \rightarrow g$
        \item $\bar{v}^{\beta_n}$ converges pointwise to some function $\bar{v}$.
    \end{itemize}    
    Finally, since $\sup_{(\xi,\xi^\prime)\in S^2} v^\beta_{\pi_\theta}(\xi) - v^\beta_{\pi_\theta}(\xi^\prime) = C_\theta$ is uniformly bounded so is $\bar{v}$:
    \begin{equation*}
        \sup_{(\xi,\xi^\prime)\in S^2} \bar{v}(\xi) - \bar{v}(\xi^\prime) \leq \sup_{(\xi,\xi^\prime)\in S^2} \sup_{n\geq 1} \bar{v}^{\beta_n}(\xi) - \bar{v}^{\beta_n}(\xi^\prime) \leq \sup_{n\geq 1} \sup_{(\xi,\xi^\prime)\in S^2}  \bar{v}^{\beta_n}(\xi) - \bar{v}^{\beta_n}(\xi^\prime) \leq 2 C_\theta.
    \end{equation*}
    We are now left to check that the pair $(g,\bar{v})$ satisfies the Bellman equation in Cond.~\ref{cond:implicit.bellman}. It relies on the following lemma (Lemma~3 in \citet{platzman1980optimal}).
    
    \begin{lemma}
    If $\bar{v}^{\beta_n}$ converges to $\bar{v}$ pointwise, then $\mathcal{T}_1(\bar{v}^{\beta_n})$ converges to $\mathcal{T}_1(\bar{v})$ pointwise.
    \label{le:platzman}
    \end{lemma}
    \begin{proof}
    We provide the proof for the sake of completeness. The objective is to prove that for an arbitrary fixed $\xi \in S$, $\epsilon > 0$, there exits a constant $M$ such that $|[\mathcal{T}_1(\bar{v}^{\beta_n}) -\mathcal{T}_1(\bar{v})](\xi)| < \epsilon$ for all $n \geq M$.\\
    Let $\xi \in S$ be an arbitrary state, and define $$S_\xi = \{ \xi^\prime \in S \text{ s.t. } \mathbb{P}_\theta(\xi_{t+1} = \xi^\prime | \xi_t = \xi, A_t = \pi_{\theta}(\xi)) > 0 \}.$$ Notice that since $A_t$ is fully determined by $\xi$, so is $\xi_{t+1}^n$ and $S_\xi$ is a finite non-empty set of state. Thus, there exists $M$ such that $|\bar{v}^{\beta_n}(\xi^\prime) - \bar{v}(\xi^\prime)| < \epsilon$ for all $\xi^\prime \in S_\xi$, $n \geq M$. Finally, it leads to
    \begin{equation*}
        |[\mathcal{T}_1(\bar{v}^{\beta_n}) -\mathcal{T}_1(\bar{v})](\xi)| \leq \mathbb{E}_\theta \left( | \bar{v}^{\beta_n}(\xi^\prime) - \bar{v}(\xi^\prime)| \hspace{1mm} \big | \hspace{1mm} \xi \right) \leq \max_{\xi^\prime\in S_\xi} | \bar{v}^{\beta_n}(\xi^\prime) - \bar{v}(\xi^\prime)| < \epsilon,
    \end{equation*}
    which proves the desired result.
    \end{proof}
    Finally, the uniform boundedness of $\bar{v}^{\beta_n}$ implies 
    \begin{equation*}
        |\mathcal{T}_{\beta_n}(\bar{v}^{\beta_n}) - \mathcal{T}_{1}(\bar{v}^{\beta_n})| \rightarrow 0,
    \end{equation*}
    which in addition to Lemma~\ref{le:platzman} ensures that 
    \begin{equation*}
        |\mathcal{T}_{\beta_n}(\bar{v}^{\beta_n}) - \mathcal{T}_{1}(\bar{v}) | \leq |\mathcal{T}_{\beta_n}(\bar{v}^{\beta_n}) - \mathcal{T}_{1}(\bar{v}^{\beta_n})| + |\mathcal{T}_{1}(\bar{v}^{\beta_n})- \mathcal{T}_{1}(\bar{v}) | \rightarrow 0.
    \end{equation*}
    Taking the limit in Eq.~\ref{eq:proof.bellman.1} concludes the proof
    \begin{equation*}
        \bar{v}^{\beta_n} + (1- \beta_n) v^{\beta_n}_{\pi_\theta}(\bar{\xi}) - \mathcal{T}_{\beta_n}(\bar{v}^{\beta_n}) = 0 \quad \Rightarrow \quad 
        \bar{v} + g - \mathcal{T}_1(\bar{v}) = 0.
    \end{equation*}
\end{enumerate}

\section{Regret Bound Proofs}
\label{app:regret.proofs}
In this section, we provide full proofs that are sketched in Sec. \ref{sec:regret}.
\subsection{Regret Decomposition}

Let $(\theta_i,\pi_i)$ be the sampled parameter-policy pair used in episode $i$. From Eq.~\ref{eq:bellman}, one has
\begin{align*}
\sum_{t=t_i}^{t_{i+1} - 1} 
    r_{\thetaStar}(\xi_t, A_t)
&=\sum_{t=t_i}^{t_{i+1} - 1} { 
    [r_{\theta_i} + (r_{\thetaStar} - r_{\theta_i})](\xi_t, \pi_i(\xi_t))
}\\
&= \sum_{t=t_i}^{t_{i+1} - 1} 
[
J_{\pi_i}(\theta_i) + v_{\theta_i}(\xi_t)
-\E_{\theta_i} [v_{\theta_i}(\xi^\prime) | \pi_i, \xi_t]
+ (r_{\thetaStar} - r_{\theta_i})(\xi_t, \pi_i(\xi_t))].
\end{align*}

Using this, we can rewrite the frequentist regret by
\begin{align*}
R(T;\thetaStar) 
&= J_{\piStar}(\thetaStar)\cdot T - \E_{\thetaStar}\sum_{i=1}^{M_T}\sum_{t=t_i}^{t_{i+1} - 1} r_{\thetaStar}(\xi_t, A_t) \\
&=: R_0 + R_1 + R_2 + R_3,
\end{align*}
where
\begin{align*}
R_0 
&= J_{\piStar}(\thetaStar)\cdot T 
- \E_{\thetaStar}\sum_{i=1}^{M_T}  J_{\pi_i}(\theta_i) \cdot T_i\\
R_1
&= \E_{\thetaStar}\sum_{i=1}^{M_T}\sum_{t=t_i}^{t_{i+1} - 1} 
v_{\theta_i}(\xi_{t+1}) - v_{\theta_i}(\xi_t) \\
R_2
&=\E_{\thetaStar}\sum_{i=1}^{M_T}\sum_{t=t_i}^{t_{i+1} - 1} 
\E_{\theta_i} [v_{\theta_i}(\xi^\prime) | \pi_i, \xi_t]
 - v_{\theta_i}(\xi_{t+1}) \\
R_3
&= \E_{\thetaStar}\sum_{i=1}^{M_T}\sum_{t=t_i}^{t_{i+1} - 1}
(r_{\theta_i} - r_{\thetaStar})(\xi_t, \pi_i(\xi_t)).
\end{align*}


\subsection{Confidence Set}
We begin with a useful result that is induced by Asm.~\ref{asm:mixing}.
\begin{proposition}\label{prop:mixing.times}
For any arm $k \in [K]$ and $\theta \in \Theta$, let  
$$
p_\theta(s^\prime;k,s,n) = \mathbb{P}_\theta(s^t_k = s^\prime | s^{t-n}_k = s)
$$ 
and $p_\theta(k, s, n)$ be the corresponding distribution over $S_k$.
For any $\epsilon >0$ and $n,n^\prime > \log_2(1/\epsilon) \Tmix(\frac{1}{4})$, we have
\begin{equation*}
    \|p_\theta(k,s,n) - p_\theta(k,s,n^\prime)\|_1 \leq 2 |S_k| \epsilon.
\end{equation*}
\end{proposition}
\begin{proof}
For any $n \geq 1$, we can write
$$
p_\theta(k,s,n) = \big(P^{passive}_k\big)^{n-1} P^{active} e^s,
$$
where $e^s$ is a binary vector of size $|S_k|$ with $1$ on the $s$ entry and $0$ elsewhere. 
We can deduce
\begin{equation*}
\begin{aligned}
     \|p_\theta(k,s,n) - p_\theta(k,s,n^\prime)\|_1 
     &\leq |S_k| \max_{s \in S_k} \|\Big(\big(P^{passive}_k\big)^{n-1} - \big(P^{passive}_k\big)^{n^\prime -1}\Big)e^s\|_1 \\
     &\leq 2 |S_k| \max_{s \in S_k} \|p_k^{T^{mix}}(s) - p_k\|_1 \\
     &\leq 2 |S_k|\epsilon,
\end{aligned}
\end{equation*}

where we used the fact $\Tmix_k(\epsilon) \leq \log_2(1/\epsilon) \Tmix(\frac{1}{4})$
, discussed by \citet[Eq. 1]{ortner2012regret}.

\end{proof}

Now we prove Lemma \ref{lemma:onpolicyerror}.

\newtheorem*{new.lemma:onpolicyerror}{Lemma~\ref{lemma:onpolicyerror}}
\begin{new.lemma:onpolicyerror}
On the high-probability event $\thetaStar \in \cap_{i\leq M_T} \Theta_i$, we can show
\begin{equation*}
\Delta_T \leq 12
\sqrt{
    N\Tmix T \log 1/\delta
}
\sum_{k=1}^K |S_k|.
\end{equation*}
\end{new.lemma:onpolicyerror}
\begin{proof}
We work on the high-probability event where $\thetaStar \in \Theta_i$ for all $i\leq M_T$. 
Thus, from Lemma~\ref{lemma:confidenceProb} we have 
$||(\hat{p}_{t_i} - p_{\thetaStar})(\zeta)||_1 \le c_i(\zeta)$ for all $\zeta$. Hence, we obtain
\begin{equation*}
    \Delta_T \le \sum_{i=1}^{M_T}\sum_{t=t_i}^{t_{i+1} - 1}
\sum_{\text{active arms } k}{
    c_i (k, \sigma^t_k, n^t_k)}.
\end{equation*}
By the second stopping criterion of TSDE, we have $N_t(\zeta) \le 2 N_{t_i}(\zeta)$ for all $t$ in episode $i$. Using this, we can write
\begin{align*}
\sum_{i=1}^{M_T}\sum_{t=t_i}^{t_{i+1} - 1} 
\sum_{\text{active arms } k}{
    c_i (k, \sigma^t_k, n^t_k) 
} 
&\le
\sum_{i=1}^{M_T}\sum_{t=t_i}^{t_{i+1} - 1} 
\sum_{\text{active arms } k}{
    \sqrt{
        \frac{
            16 |S_k| \log 1/\delta
        }{
            1 \vee N_{t}(k, \sigma^t_k, n^t_k)
        }
    }
}\\
&=
\sum_{k=1}^K{
    \sumtt{
        \ind(A_{t, k} = 1)
        \sqrt{
            \frac{
                16 |S_k| \log 1/\delta
            }{
                1 \vee N_{t}(k, \sigma^t_k, n^t_k)
            }
        }
    }
}.
\end{align*}
For each $\zeta = (k, s, n)$, it appears in the above summation exactly $N_{T+1}(\zeta)$ times. That is to say, the above equation can be written as
\begin{equation}
\label{eq:r2_4}    
\sum_{\zeta \in Z}{
    \sqrt{
        16|S_k| \log 1/\delta
    }
    \cdot
    \sum_{j=1}^{N_{T+1}(\zeta)}{
        \frac{1}{\sqrt{1 \vee (j-1)}}
    }
} \le
\sum_{\zeta \in Z}{
    12\sqrt{
        |S_k|N_{T+1}(\zeta) \log 1/\delta
    }
}.
\end{equation}
The number of $\zeta = (k, s, n)$ for a fixed $k$ is bounded by
$|S_k|\Tmix$.
Also, we have $\sum_{\zeta\in Z} N_{T+1}(\zeta) = NT$. 
The Cauchy-Schwartz inequality provides 
\begin{align*}
\sum_{s \in S_k} \sum_{n=1}^{\Tmix}{
    \sqrt{N_{T+1}(k, s, n)}
}
&\le
\sqrt{
    |S_k|\Tmix NT
}.
\end{align*}
Finally, we obtain
\begin{equation*}
    \Delta_T \le 12
\sqrt{
    N\Tmix T \log 1/\delta
}
\sum_{k=1}^K |S_k|.
\end{equation*}
\end{proof}

\subsection{Bounding $R_0$ and $R_1$}
\newtheorem*{new.lemma:r0}{Lemma~\ref{lemma:r0}}
\begin{new.lemma:r0}[\citet{ouyang2017learning}, Lemma~3 and 4]
\begin{equation*}
    \E_{\thetaStar \sim Q} R_0 \le N \cdot \E_{\thetaStar \sim Q} M_T,
\end{equation*}
where $M_T$ is the total number of episodes until time $T$.
\end{new.lemma:r0}
\begin{proof}
By definition in Eq.~\ref{eq:averageReward},
we have $0 \le J_\pi(\theta) \le N$ for all $\pi$ and $\theta$.
For ease of analysis, 
let us write
\[
J_{\piStar}(\thetaStar) 
=
N - J^\star
\text{ and }
J_{\pi_i}(\theta_i)
=
N - J_i.
\]
Since $M_T \le T$ almost surely, we can rewrite 
\begin{align*}
R_0
&=
J_{\piStar}(\thetaStar)\cdot T 
- \E_{\thetaStar}\sum_{i=1}^{T}{
    \ind(t_i \le T) J_{\pi_i}(\theta_i) \cdot T_i
}\\
&=
\E_{\thetaStar}\sum_{i=1}^{T}{
    \ind(t_i \le T) J_i \cdot T_i
}
-J^\star \cdot T 
.
\end{align*}
Due to the first stopping criterion of TSDE, we have 
$T_i \le T_{i-1}+1$ for all $i$. 
Using this, we can deduce
\begin{align*}
R_0 
\le 
\E_{\thetaStar}\sum_{i=1}^{T}{
    \ind(t_i \le T) J_i \cdot (T_{i-1}+1)
}
-J^\star \cdot T 
.
\end{align*}
In the meantime, note that 
$\ind(t_i \le T) J_i \cdot (T_{i-1}+1)$
is a $\Hc_{t_i}$-measurable function of $\theta_i$. 
Thus Lemma \ref{lemma:expectationIdentity} implies 
\begin{align*}
\E \ind(t_i \le T) J_i \cdot (T_{i-1}+1)
=
\E \ind(t_i \le T) J^\star \cdot (T_{i-1}+1).
\end{align*}
Using this, we obtain
\begin{align*}
\E_{\thetaStar \sim Q} R_0
&\le 
\E \sum_{i=1}^{T}{
    \ind(t_i \le T) J^\star \cdot (T_{i-1}+1)
}
-J^\star \cdot T 
= \E J^\star \cdot M_T.
\end{align*}
Since $J^\star \le N$ almost surely, this completes the proof.
\end{proof}

\newtheorem*{new.lemma:r1}{Lemma~\ref{lemma:r1}}
\begin{new.lemma:r1}
\begin{equation*}
R_1 \le H \cdot \E M_T.
\end{equation*}
\end{new.lemma:r1}
\begin{proof}
For a fixed episode $i$, the telescope rule gives
\begin{align*}
\sum_{t=t_i}^{t_{i+1} - 1} 
v_{\theta_i}(\xi_{t+1}) - v_{\theta_i}(\xi_t)
=
v_{\theta_i}(\xi_{t_{i+1}}) - v_{\theta_i}(\xi_{t_{i}}),
\end{align*}
which is less than $H$ by the assumption. 
Summing over the episodes concludes the argument.
\end{proof}

\subsection{Bounding $R_2$ and $R_3$.}

Before delving into bounding $R_2$ and $R_3$, we record a technical lemma, which generalizes Lemma $7$ by \citet{jung2019regret}.

\begin{lemma}
\label{lemma:technical}
Suppose $a_k$ and $b_k$ are probability distributions over a set 
$[n_k]$ for $k \in [K]$.
Then we have
\[
\sum_{x \in \otimes_{k=1}^K [n_k]}{
    |\prod_{k=1}^K a_{k, x_k} - \prod_{k=1}^K b_{k, x_k}|
}
\le 
\sum_{k=1}^K ||a_k - b_k||_1 .
\]
\end{lemma}
\begin{proof}
Fix a vector $x$. 
For simplicity,
let $\alpha_k = a_{k, x_k}$,
$\beta_k = b_{k, x_k}$,
and $\delta_k = |\alpha_k - \beta_k|$.
We may write
\begin{align*}
|\prod_{k=1}^K \alpha_k - \prod_{k=1}^K \beta_k| 
&\le 
(\prod_{k=1}^{K-1} \alpha_k )|\alpha_K - \beta_K|
+
|\prod_{k=1}^{K-1} \alpha_k - \prod_{k=1}^{K-1} \beta_k|\beta_K\\
&=
(\prod_{k=1}^{K-1} \alpha_k )\delta_K
+
|\prod_{k=1}^{K-1} \alpha_k - \prod_{k=1}^{K-1} \beta_k|\beta_K\\
&\le \cdots \\
&\le
\sum_{k=1}^K{
    (\prod_{j=1}^{k-1}\alpha_j)
    \delta_k
    (\prod_{j=k+1}^{K}\beta_j)
} \\
&=
\sum_{k=1}^K{
    (\prod_{j=1}^{k-1}a_{j, x_j})
    |a_{k, x_k} - b_{k, x_k}|
    (\prod_{j=k+1}^{K}b_{j, x_j})
}.
\end{align*}
When summing the last term for all possible vectors $x$, 
the coefficient of 
$|a_{k, x_k} - b_{k, x_k}|$
becomes $1$ because $a_k$ and $b_k$ are probability distributions.
Then we get the desired inequality.
\end{proof}

\newtheorem*{new.lemma:r2}{Lemma~\ref{lemma:r2}}
\begin{new.lemma:r2}
$R_2$ satisfies the following bound
\[
R_2 \le 28H \sum_{k=1}^K |S_k|
\sqrt{
    N\Tmix T \log (\Tmix T)
}.
\]
\end{new.lemma:r2}
\begin{proof}
In episode $i$, $\xi_{t+1}$ evolves from $\xi_t$ on the system $\thetaStar$ with the action $A_t = \pi_i(\xi_t)$.
From this, we can rewrite
\begin{align*}
R_2
&=
\E_{\thetaStar}\sum_{i=1}^{M_T}\sum_{t=t_i}^{t_{i+1} - 1} 
(\E_{\theta_i} - \E_{\thetaStar}) [
    v_{\theta_i}(\xi^\prime) | \pi_i, \xi_t
].
\end{align*}
Since $|v_{\theta_i}(\xi^\prime)|\le H$, the individual difference becomes
\begin{align*}
\begin{split}
\sum_{\xi^\prime \in S}& { 
    (
        \prob_{\pi_i(\xi_t)}(\xi_t, \xi^\prime | \theta_i)
        -
        \prob_{\pi_i(\xi_t)}(\xi_t, \xi^\prime | \thetaStar)
    )
    v_{\theta_i}(\xi^\prime)
}\le
H\sum_{\xi^\prime \in S} { 
    |
        \prob_{\pi_i(\xi_t)}(\xi_t, \xi^\prime | \theta_i)
        -
        \prob_{\pi_i(\xi_t)}(\xi_t, \xi^\prime | \thetaStar)
    |
}.
\end{split}
\end{align*}
Once the action $\pi_i(\xi_t)$ is fixed, 
$\xi^n_t = (n^t_1, \cdots, n^t_K)$ evolves in a deterministic manner.
Only $\sigma^t_k$ for the active arms $k$ will be updated. 
Then we may write
\begin{align*}
\prob_{\pi_i(\xi_t)}(\xi_t, \xi^\prime | \theta)
=
\prod_{\text{active arms }k}{
    p_\theta(\sigma^\prime_k ; k, \sigma^t_k, n^t_k)
},
\end{align*}
where $p_\theta(k, s, n)$ is defined earlier in the section.
Using Lemma \ref{lemma:technical}, we obtain
\begin{align}
\begin{split}
\label{eq:r2_2}
\sum_{\xi^\prime \in S} &{ 
    |
        \prob_{\pi_i(\xi_t)}(\xi_t, \xi^\prime | \theta_i)
        -
        \prob_{\pi_i(\xi_t)}(\xi_t, \xi^\prime | \thetaStar)
    |
}\le
\sum_{\text{active arms } k}{
    ||
        (p_{\theta_i} - p_{\thetaStar})(k, \sigma^t_k, n^t_k)
    ||_1.
}
\end{split}
\end{align}
If $\thetaStar, \theta_i \in \Theta_i$, we can apply Lemma~\ref{lemma:onpolicyerror} to upper bound the cumulative sum. If not, the entire summation is bounded by $2$. 
From these and Lemma~\ref{lemma:confidenceProb} and~\ref{lemma:onpolicyerror}, we obtain 
\begin{align}
\begin{split}
\label{eq:r2_3}
&~~~R_2
\le
H (R_{2}^0  + R_{2}^1),
\text{ where}\\
R_{2}^0
&= 24
\sqrt{
    N\Tmix T \log 1/\delta
}
\sum_{k=1}^K |S_k| , \\
R_{2}^1
&=
4\delta \Tmix T \sum_{k=1}^K |S_k|.
\end{split}
\end{align}
We finish the proof by setting $\delta = \frac{1}{\Tmix T}$ in Eq.~\ref{eq:r2_3}.
\end{proof}

\newtheorem*{new.lemma:r3}{Lemma~\ref{lemma:r3}}
\begin{new.lemma:r3}
$R_3$ satisfies the following bound
\[
R_3 \le 28 \sum_{k=1}^K |S_k|
\sqrt{
    N\Tmix T \log (\Tmix T)
}.
\]
\end{new.lemma:r3}
\begin{proof}
We begin by investigating the individual term
\begin{align*}
(r_{\theta_i} - r_{\thetaStar})(\xi_t, \pi_i(\xi_t))
&=
\sum_{\text{active arms }k}{
    (\E_{\theta_i} - \E_{\thetaStar})[
        r_k(s^t_k)
        |
        \xi_t, \pi_i(\xi_t)
    ]
}\\
&=
\sum_{\text{active arms }k}{
    \sum_{s^\prime \in S_k}{
        r_k(s^\prime)
        (p_{\theta_i} - p_{\thetaStar})(s^\prime ; k, \sigma^t_k, n^t_k)
    }
}\\
&\le
\sum_{\text{active arms }k}{
    \sum_{s^\prime \in S_k}{
        ||
            (p_{\theta_i} - p_{\thetaStar})(k, \sigma^t_k, n^t_k)
        ||_1
    }
},
\end{align*}
where the last inequality holds by the assumption $r_k(s_k) \le 1$.
The last term actually appears in Eq. \ref{eq:r2_2} from the proof of Lemma \ref{lemma:r2}, 
and we can use the same argument to obtain the desired bound.
\end{proof}

\end{appendices}

\end{document}